\documentclass[letterpaper, 10 pt, conference]{ieeeconf}  % Comment this line out if you need a4paper

\IEEEoverridecommandlockouts                              % This command is only needed if 
\usepackage{graphics} % for pdf, bitmapped graphics files
\usepackage{epsfig} % for postscript graphics files
\usepackage{mathptmx} % assumes new font selection scheme installed
\usepackage{times} % assumes new font selection scheme installed
\usepackage{amsmath} % assumes amsmath package installed
\usepackage{amssymb}  % assumes amsmath package installed

\usepackage[utf8]{inputenc}
\usepackage[shortcuts]{extdash} % For breaking at dashes
\usepackage{breakcites}
\usepackage[ruled,noend]{algorithm2e}
\usepackage{color}
\usepackage{url}
\usepackage{tabularx}
\let\labelindent\relax
\usepackage{enumitem}

\newtheorem{myTheorem}{Theorem}
\newtheorem{myDefinition}{Definition}

\newcommand{\MyAlgoCapSpace}{0.5em}
\newcommand{\MyAlgoDecSpace}{0.25em}

\newcommand{\Datom}{\textit{Datom}\xspace}
\newcommand{\Datoms}{\textit{Datoms}\xspace}

\title{\LARGE \bf
	Datom: A Deformable modular robot for building self-reconfigurable programmable matter
}

\author{Benoît Piranda$^{\dagger}$ and Julien Bourgeois$^{\dagger}$% <-this % stops a space
	\thanks{$^{\dagger}$All authors are with Univ. Bourgogne Franche-Comt\'e, FEMTO-ST Institute, CNRS, 1~cours Leprince-Ringuet, 25200, Montb\'eliard, France.
		{\tt\small \{first\}.\{last\}@femto-st.fr}}%
}

\begin{document}

	\maketitle
	\thispagestyle{empty}
	\pagestyle{empty}

	%%%%%%%%%%%%%%%%%%%%%%%%%%%%%%%%%%%%%%%%%%%%%%%%%%%%%%%%%%%%%%%%%%%%%%%%%%%%%%%%
	\begin{abstract}
		Moving a module in a modular robot is a very complex and error-prone process. Unlike in swarm, in the modular robots we are targeting,  the moving module must keep the connection to, at least, one other module. In order to miniaturize each module to few millimeters, we have proposed a design which is using electrostatic actuator. However, this movement is composed of several attachment, detachment creating the movement and each small step can fail causing a module to break the connection.
		The idea developed in this paper consists in creating a new kind of deformable module allowing a movement which keeps the connection between the moving and the fixed modules.
		We detail the geometry and the practical constraints during the conception of this new module.
		We then validate the possibility of movement for a module in an existing configuration. This implies the cooperation of some of the modules placed along the path and we show in simulation that it exists a motion process to reach every free positions of the surface for a given configuration.
		
	\end{abstract}
	
	\section{Introduction}
	The idea of designing hardware robotic modules able to be attached together has given birth to the field of modular robotics and when these modules can move by themselves they are named Modular Self-reconfigurable Robots (MSR) \cite{yim_2007_modular}\cite{stoy_2010_self} also named earlier as metamorphic robotic systems \cite{chirikjian_1994_kinematics} or cellular robotic systems \cite{fukuda_1989_communication}. There are five families of MSR namely: lattice-based when modules are aligned on a 3D lattice, chain-type when the modules are permanently attached through an articulation, forming a chain or more rarely a tree, hybrid which is a mix between lattice-based and chain-type, mobile when each module can move autonomously and more recently continuous docking \cite{swissler_2018_fireant} where latching can be made in any point of the module. 
	Since then, there have been many robots proposed and built by the community using different scales of modules and different latching and moving technologies. However, none of them have succeeded to reach a market.
	
	Instead of building a multi-purposes modular robot, and then trying to apply it for a given task, we start with the application and we propose the design of the modular robot to fit this application. Our objective is to build programmable matter~\cite{bourgeois_2016_programmable} which is a matter which can change one or several of its physical properties, more likely its shape, according to an internal or an external action. Here, programmable matter will be constructed using a MSR, i.e. a matter composed of mm-scale robots, able to stick together and turn around each other as it has been described in the Claytronics project \cite{goldstein_2004_claytronics}. The Programmable Matter Project\footnote{http://projects.femto-st.fr/programmable-matter/} is a sequel of the Claytronics project and reuses most of its ideas and concepts. The requirements for each module are the following: mm-scale, being able to move in 3D, compute and communicate with their neighbors and the idea is to have thousands of them all linked together.
	Moving in 3D is the most complicated requirement as it needs a complex trade off between several parameters during the design phase. For example, moving requires power and, therefore, power storage, which adds weight to the module, the trade off being between having more power by adding more power storage and having a module as light as possible for easing the movement.
	We are currently building and testing a quasi-spherical module we designed \cite{piranda_2018_designing}. This module rolls on another module using electrostatic electrodes. This way of moving creates uncertainty in the success of the movement as it is a complex sequence of repulsing/attaching/detaching actuations and we would like to study a movement where the moving module always stay latched to the pivot module.

	The idea that drives this work is to design a motion process which never disconnects the moving and the fixed modules. We propose to define a deformable module named Deformable Atom \Datom, as a reference to the Claytronics Atom, Catom. % that can be associate with other identical modules to create a volume of programmable matter.
	Each module is strongly connected to neighbors in the Face-centered cubic 
	(\textsc{FCC}) lattice with large connectors (drawn in red in all following figures). 
	Two connected modules must deform their shapes to align future latched connectors while the previous connection is maintained.
	When new connectors are aligned they are strongly attached and the previous connection is released. Finally, the two modules return to their original shape.
	
	Figure~\ref{fig:motion} shows the decomposition of the movements of a mobile module $B$ moving around a fixed module $A$ to go from the position shown in Figure~\ref{fig:motion}.a to the position shown in Figure~\ref{fig:motion}.f. We consider that connectors $B_1$ and $A_{10}$ are initially attached. 
	In Figure~\ref{fig:motion}.a and Figure~\ref{fig:motion}.f $A$ and $B$ are not moving while in Figure~\ref{fig:motion}.d they are under actuation and deformed. 
	During motion, simultaneous deformations of the two modules allow to maintain the connection between $B_1$ and $A_{10}$. At the middle of the deformation process (see Figure~\ref{fig:motion}.d), four connectors of $B$ are in front of four connectors of $A$: (${(A_{10},B_1),(A_3,B_2),(A_2,B_3),(A_5,B_4)}$), but only one couple $(A_{10},B_1)$ is still attached. In this case, four different motions can be used to reach four different positions. To move to the final destination, connectors $A_3$ and $B_2$ are then attached and connectors $(A_{10},B_1)$ are released. An mirrored deformation from the previous ones moves module $B$ to its final position.
	
	\section{Related works}

	%La litterature propose de nombreux robots autonomes pour la matière programmable \cite{} mais nous nous intéressons ici aux robots modulaires pouvant se déplacer dans une grille en utilisant des déformations internes.
	
	Many solutions are available in the literature to create robots. In the Programmable Matter context, we try to design robots that can scale down to small size, using low power for processors and actuators. In this paper, we are interested in solutions that ensure that a motion of a module allow to reach a cell of the lattice.
	
	%Les crystallines sont des robots déformables développés par Rus et al~\cite{rus_2001_crystalline}, ils peuvent s'allonger et se contracter de façon à doubler leur dimensions afin de produire le déplacement d'un voisin par glissement. Ils sont organisés en méta modules de 4x4 unités. L'attachement se fait par un mouvement mécanique "lock and key" au niveau des faces de contact.
	
	Crystalline Robots~\cite{rus_2001_crystalline}, developed by Rus et al. in 2001 is an interesting solution. These robots can move relatively to each other by expanding and contracting. A robot can move a neighbor by doubling its length along $\vec{x}$ and $\vec{y}$ axes. These robots are grouped in meta-modules of 4x4 units placed in a 2D square grid. Robot to robot attachment is made by a mechanical system called "lock and ley" located on the square connected faces.
	
	In~\cite{suh_2002_telecubes} Suh et al propose the Telecube, a cubic robot able to move in a cubic lattice. Similarly to previous work, Telecube can shrink using internal motors to move a neighbor. Telecube are grouped in meta-modules made of $2 \times 2 \times 2$ units. The six arms are terminated by sensors to detect neighbors and electro-permanent magnets connect the arm of the neighboring module.
	
	The \textit{Catom} model presented in~\cite{piranda_2018_designing} is a robot that can move in a \textsc{FCC} lattice in rolling on the border of its neighbors. It uses electrostatic actuators, both for latching on planar connectors and rolling around cylindrical parts separating connectors.

	Table~\ref{tab:comparison} shows a comparison of these robots and the \Datom model.
	\begin{table}[ht]
		\centering
		\caption{Comparison}
		\label{tab:comparison}
		\begin{tabularx}{0.48\textwidth}{>{\setlength\hsize{0.8\hsize}}X>{\setlength\hsize{0.6\hsize}}X>{\setlength\hsize{1.7\hsize}}X>{\setlength\hsize{1.3\hsize}}X>{\setlength\hsize{1.0\hsize}}X>{\setlength\hsize{0.6\hsize}}X} %>{\setlength\hsize{0.5\hsize}}X|>{\setlength\hsize{1.5\hsize}}X|}
			\hline
			Robot & Lattice  & Strong attachment & MetaModule granularity & Tunnelling & Motion \\ \hline
			Cristalline & Square & Yes Mecanical & 4x4 & Yes & slide  \\ 
			Telecube & Cubic & Yes Magnetic  & 2x2x2  & Yes & slide \\ 
			Catom & \textsc{FCC} & No electrostatic & 1 & No & roll \\ 
			Datom & \textsc{FCC} & Yes & 1 & No & turn \\ 
			\hline 
		\end{tabularx}
	\end{table}
	
	\begin{figure}
		\center
		\includegraphics[width=0.48\textwidth]{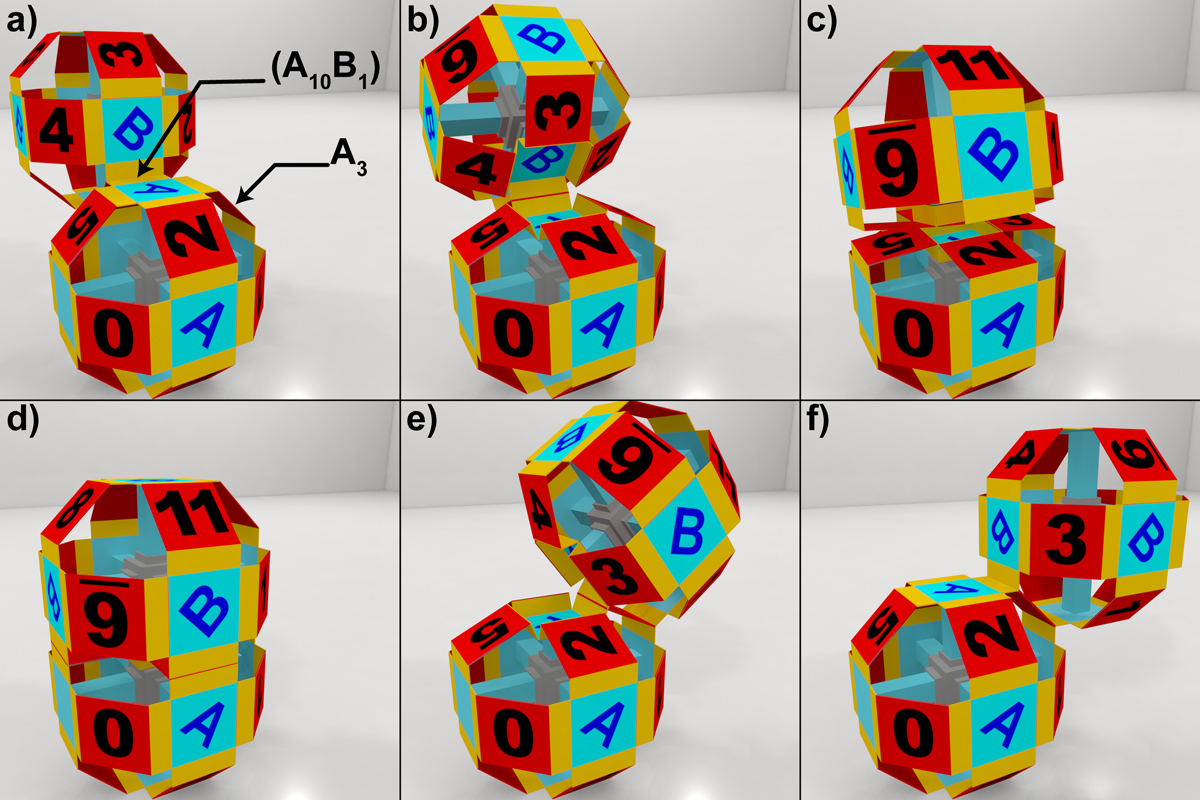}
		\caption{6 steps of the motion of module B around the fixed module A.}
		\label{fig:motion}%
	\end{figure}
	
	\section{The \Datom model}
	
	\subsection{Theoretical geometry of the deformable module}
	The shape of the module is deduced from the shape of the \textit{catom} proposed in~\cite{piranda_2018_designing}. From this initial geometry, we retain the position of the 12 square connectors, centered at $P_i$. This positions are imposed by the placement of modules in the \textsc{FCC} lattice.
	
	\begin{equation}
	\begin{array}{l|l|l}
	P_0(r,0,0) & P_2(\frac{r}{2},\frac{r}{2},\frac{r}{\sqrt{2}}) & P_8(-\frac{r}{2},-\frac{r}{2},-\frac{r}{\sqrt{2}}) \\[1.5mm]
	P_1(0,r,0) & P_3(-\frac{r}{2},\frac{r}{2},\frac{r}{\sqrt{2}}) & P_9(\frac{r}{2},-\frac{r}{2},-\frac{r}{\sqrt{2}}) \\[1.5mm]
	P_6(-r,0,0) & P_4(-\frac{r}{2},-\frac{r}{2},\frac{r}{\sqrt{2}}) & P_{10}(\frac{r}{2},\frac{r}{2},-\frac{r}{\sqrt{2}}) \\[1.5mm]
	P_7(0,-r,0) & P_5(\frac{r}{2},-\frac{r}{2},\frac{r}{\sqrt{2}}) & P_{11}(-\frac{r}{2},\frac{r}{2},-\frac{r}{\sqrt{2}}) \\[1.5mm]
	\end{array}
	\end{equation}
	
	The size of the \Datom is given by the distance between its two opposite connectors, this diameter is equal to $2 \times r$ (where $r$ is the radius). 
	
	Electrostatic actuators produce latching forces that are proportional to the surface of the actuator. Then, maximizing the size $c$ of the square connector increases connection strength.
	We search the maximum size of connector ($c$) that allows to connect simultaneously two connectors in the deformed shape. 
	The goal is to align connector for each neighboring module to connect these connectors at the same time. 
	If we consider the plane of four coplanar connectors (see Figure~\ref{fig:motion}d for example), we can see that the maximum width of connector is the 'diagonal' length $\ell$ of the module divided by $3$. 
	
	Considering the point of view presented in Figure~\ref{fig:shape}, we can express $c=\frac{\ell}{3}$ where $\ell=\sqrt{2}(r+\frac{c}{2})$. We obtain:
	
	\begin{equation}
	c = \dfrac{2 \times r}{3\sqrt{2}-1} \approx 0.61678 \times r
	\end{equation} 
	
	In Figure~\ref{fig:shape}, connectors of length $c$ are drawn in red and the piston actuator of length $c$ is drawn in blue. Mechanical links (drawn in green) of length $e$ are placed between piston and connectors. 
	
	Figure~\ref{fig:shape}.a shows 2 connectors $C_0$ and $C_1$ viewed from the top.
	In order to align them, we propose to turn them around the $\overrightarrow{z}$ axis at points $P_0$ and $P_1$ with an angle of $+45^\circ$ for $C_0$ and $-45^\circ$ for $C_1$ as shown in Figure~\ref{fig:shape}.b.
	
	\begin{figure}
		\center
		\includegraphics[width=0.48\textwidth]{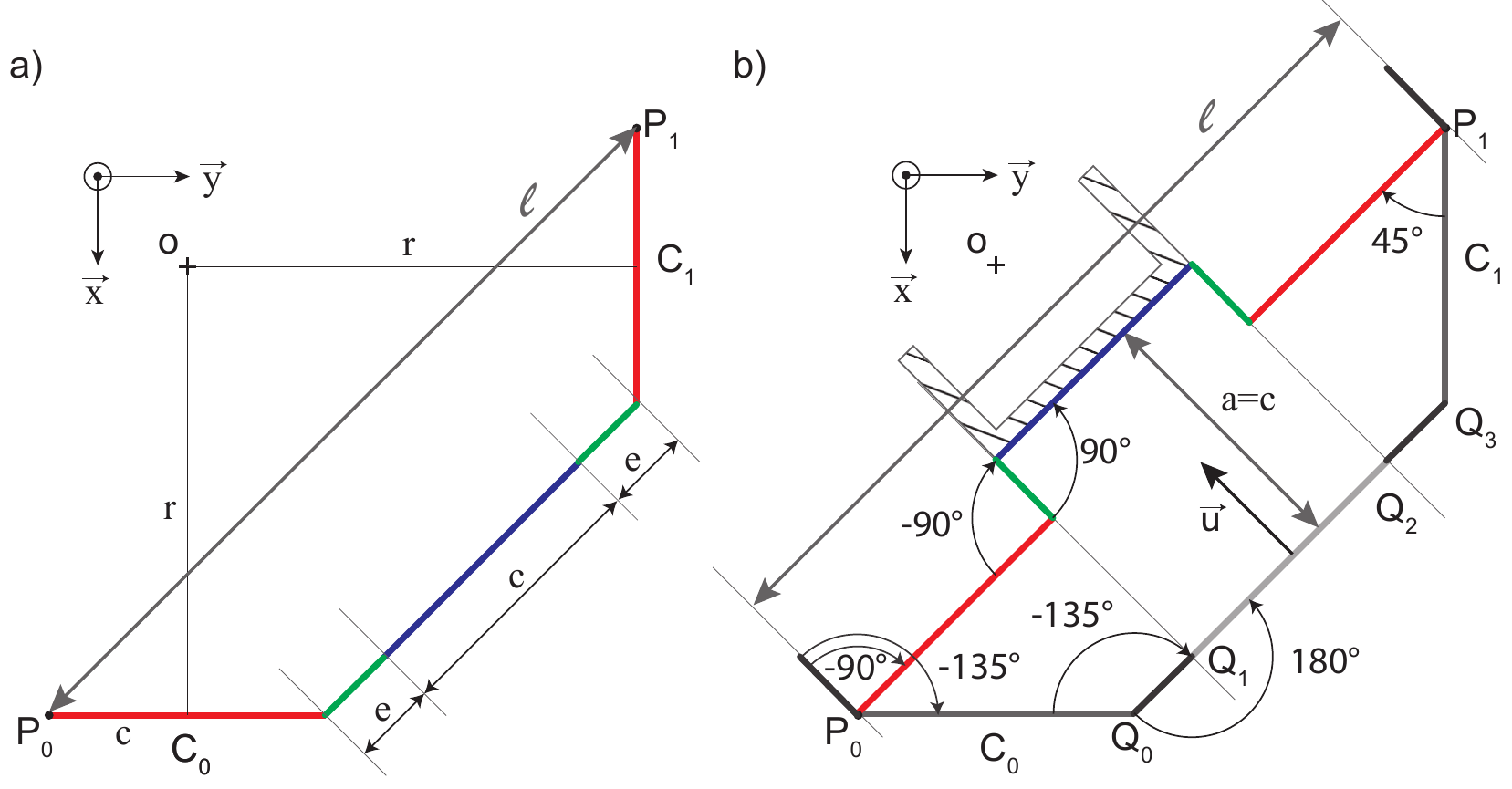}
		\caption{Size and position of each component of the robot to allow deformation. a) The rest position of the blue piston places red connector in the border of the \textsc{FCC} cells. b) Compressed position of the piston aligns connectors using green links to allow motion.}
		\label{fig:shape}%
	\end{figure}
	
	Considering Figure~\ref{fig:shape}.a, we can write a relation between $c$, $r$ and $e$ parameters:
	\begin{equation}
	r=\dfrac{c}{2}+\left(\dfrac{c}{2}+e\right)\sqrt{2}
	\end{equation}  
	
	That allows to deduce $e$ depending of the radius $r$:
	\begin{equation}
	e=r\left( \dfrac{2-\sqrt{2}}{3\sqrt{2}-1} \right) \approx 0.18065 \times r
	\end{equation}
	
	\subsection{Deformation}
	
	Considering Figure~\ref{fig:shape}.a, we can now calculate the amplitude $a$ of the piston translation to go from the rest position to the deformed one.
	\begin{equation}
	a=\frac{\sqrt{2}}{2} c +e = \frac{\sqrt{2}}{2}c + \left( \dfrac{2-\sqrt{2}}{3\sqrt{2}-1} \times \dfrac{3\sqrt{2}-1}{2}c \right) = c
	\end{equation}
	
	We obtain that the amplitude of motion of the piston is equal to the size of a connector. And it is interesting to remark that we can place a $c$ large cube in the centre of the module.

	The deformation to compress one side of the module is obtained by translating the corresponding piston along its $\overrightarrow{u}$ axis. 
	It implies that the angle of joint between links and connectors ($Q_0$) goes from $-135^\circ$ to $-90^\circ$ and angle of joint between links and piston ($Q_1$) goes from $180^\circ$ to $90^\circ$. Finally the angle of joint between fixed links and connectors ($P_0$) goes from $-135^\circ$ to $-90^\circ$ as shown in Figure~\ref{fig:shape}.b.
	
	During this deformation, only one of the 6 pistons must move in order to use the other elements as fixed supports at $P_0$ and $P_1$ points.
	
	\subsection{Creating thick elements}
	
	The theoretical shape of the \Datom is not usable as is. To create a real functional module, we must consider that connectors have a not null thickness. 
	
	Let $t$ be the thickness of the several mobiles parts of the module (connectors, link and piston). In order to place the \Datom in the \textsc{FCC} lattice, the important point is to keep the distance between two opposite connectors equal to $2 \times r$. We then define $r'=r-\frac{t}{2}$ as the corrected radius taking into account connector thickness.
	
	Using this corrected radius, we can express $c'$ and $e'$:
	\begin{align}
	c' = \frac{2 \times r'}{3\sqrt{2}-1} \\
	e' = \frac{2-\sqrt{2}}{3\sqrt{2}-1}r'
	\end{align}
	
	The construction of the link part implies that the thickness $t$ must be less than $e'$ (See Figure~\ref{fig:shape_thick}).
	We obtain that $t$ must be less than $0.19859 r$.
	
	%Moreover, amplitude of actuator is $c$, we must create cells to place the actuator when one of the parts of the \Datom is completely deformed. Finally we worked on defining a mechanical system to block actuators in max length position to get a rigid module at ease.
	
	The central part (we call it the "core") of the \Datom is a cube of $c-t$ edge size. %The function of this core consists in controlling the states of pistons, only one of the 6 pistons may be compressed at one time. 
	
	\begin{figure}
		\center
		\includegraphics[width=0.48\textwidth]{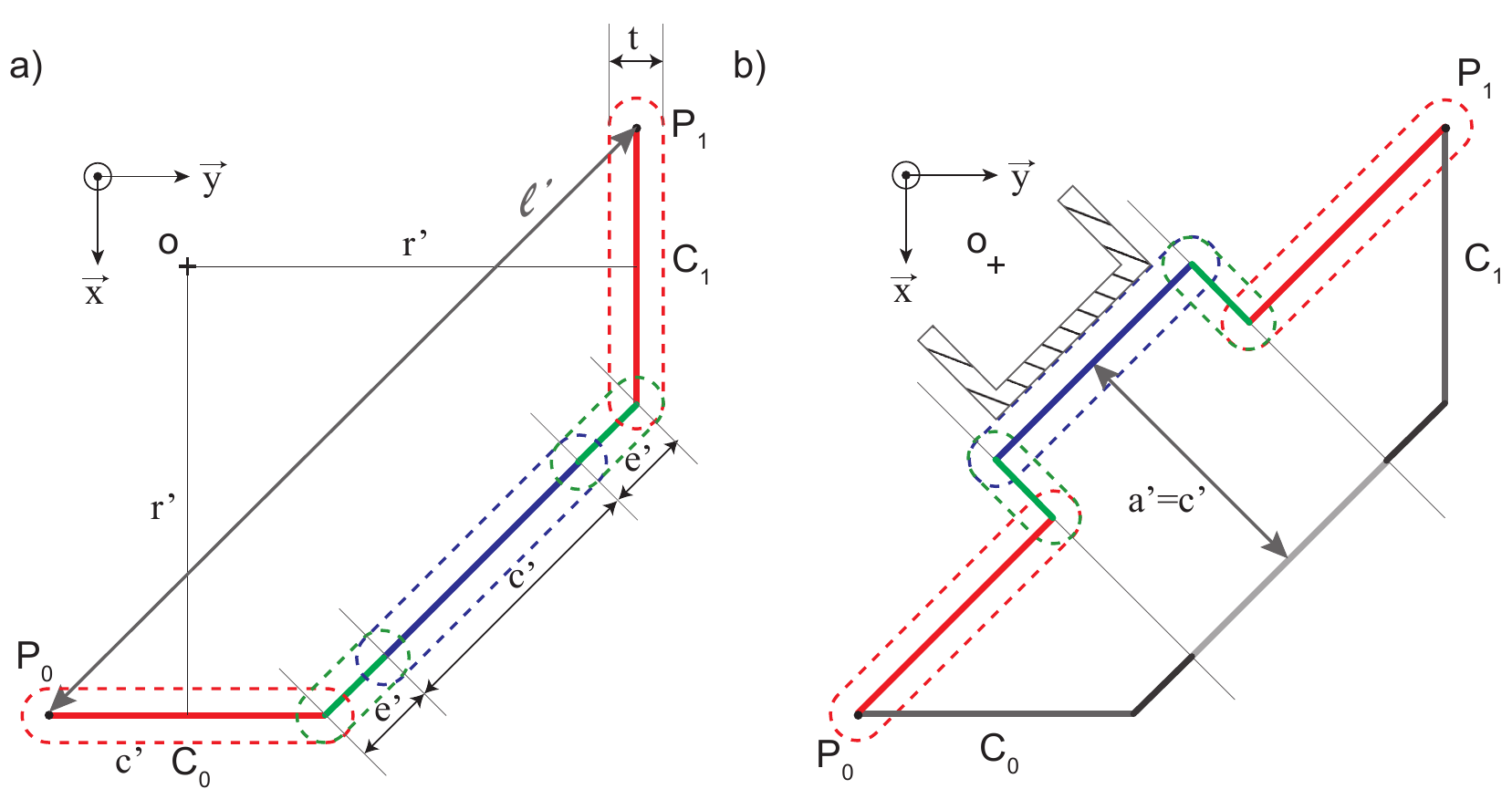}
		\caption{Size and position of components taking into account of the thickness. a) The rest position. b) Compressed position of the piston.}
		\label{fig:shape_thick}%
	\end{figure}
	
	We use rotation limits of each joint (between connector and link and between link and piston) to shape blocking plots. This blocking plots help for the stability of the whole system. 
	For example, Figure~\ref{fig:jointLinkConnector} shows blocking plots for the joint between the connector and the link parts.
	
	\begin{figure}
		\center
		\includegraphics[width=0.48\textwidth]{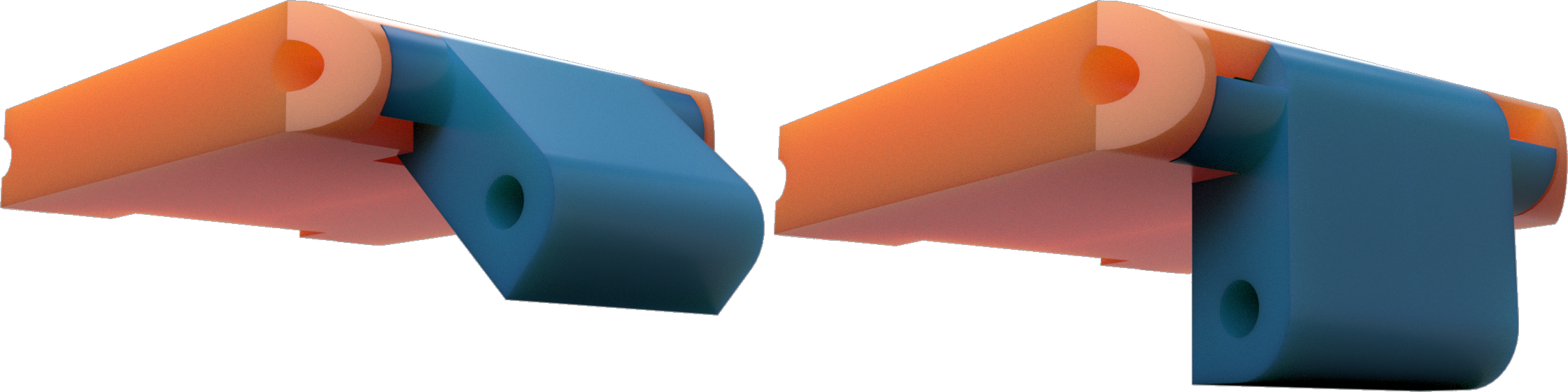}
		\caption{Angular blocking plots for joint between the connector and the link parts.}
		\label{fig:jointLinkConnector}%
	\end{figure}
	
	\subsection{Actuators}
	\subsubsection{Latching actuators}
	There are many ways to design a latching actuator and these designs use, principally, three possibilities: mechanical, electromagnetic or electrostatic. Mechanical actuators does not require power for maintaining the two modules together but they are difficult to miniaturize and slow for moving. Electromagnetic actuators require power for latching which causes heating and loss of strength. Finally, electrostatic actuators appears to be a good solution as the strength is sufficient for latching and they does not need power when latching.
	As we want to scale down our \Datom to mm-scale, the best option appears to be electrostatic actuators. We can take a design done for the cylindrical catoms~\cite{karagozler_2009_stress}.
	
	\subsubsection{Deformation actuators}
	
	\begin{figure}[b]
		\center
		\includegraphics[width=0.30\textwidth]{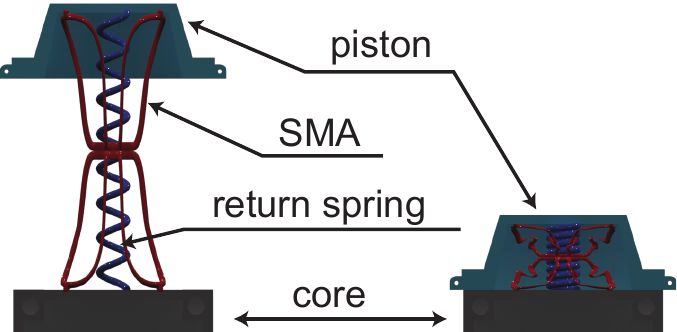}
		\caption{System of two Shape Memory Alloy (SMA) springs to actuate the piston.}
		\label{fig:spring}%
	\end{figure}
	
	%Pour l'actionnement de la déformation des datoms, deux techniques sont envisagées et devront être évaluées ultérieurement. 
	%La première consiste à utiliser des ressorts à mémoire de forme placés entre le core et chacun des pistons. Ces ressorts, doivent admettre une grande amplitude de déplacement et être couplés avec des ressorts de rappel pour revenir en position de repos (cf.~\ref{fig:spring}).
	
	In order to make the actuation of the deformation of a \Datom, we envisage two different technical solutions that must be evaluated later.
	The first one consists in placing a Shape Memory Alloy (SMA) between the piston and the core. This object is able to change his shape if warmed, it must be made in order to be long in rest mode and short in deformed state.
	This system must be coupled with a return spring that will restore the SMA in its initial shape (as shown in Figure~\ref{fig:spring}).
	
	%\TODO{La seconde solution consiste à placer des structures déformables au niveau des liens. Ces liens doivent pouvoir former un angle variant de $-90^\circ$ à $-135^\circ$ avec le piston. Avec cette solution, la problématique de commande est plus importante dans dans la sitation précédente car il faut synchroniser la déformation de 4 liens par piston pour réaliser chaque déformation.}
	
	The second solution consists in placing actuators in the intersection of the links and the pistons. These actuators must be able to change the angle between a link and a piston from $-90^\circ$ to $-135^\circ$. Electro-ribbon actuators presented by Taghavi et al.~\cite{taghavi_2018_electro-ribbon} could be adapted to create such muscles that make the deformation of the \Datom possible. In this case, the centered core is no more necessary but the synchronization of 4 actuators per piston would be complex.
	
	\section{Motion capabilities in an ensemble}
	\label{sec:six}
	
	%\begin{figure}
	%		\center
	%		\includegraphics[width=0.48\textwidth]{figures/rotations.pdf}
	%		\caption{.}
	%		\label{fig:rotations}%
	%	\end{figure}
	
	%	\subsection{Local actuations}
	%	First, we can study local effects of a \Datom displacement. 
	%Figure~\ref{fig:rotations} shows a graph with two kind of nodes, numbered square nodes represent connectors, circle nodes represents pistons. Each piston is connected to 4 connectors. In this figure we can see that if connector $0$ is connected to a neighbor connector, by actuating piston $P_0$ it can be reconnected to one of the 3 green connectors linked to $P_0$ node.
	%	If a connector is connected to a neighbor connector, by actuating one piston it can be reconnected to one of the three other connectors linked to the piston.
	%	Then a motion of a \Datom affects its attached connector.
	%	Considering an attach connector $C$ and the associated piston $P$, the reconfiguration process is decomposed into 4 steps:
	%	\begin{enumerate}
	%	\item The piston $P$ is compressed,
	%	\item One of the connectors linked to the piston but $C$ is attached,
	%	\item $C$ is detached,
	%	\item the piston $P$ is released.
	%	\end{enumerate}
	
	%	\subsection{Lattice constraints}
	We now consider a configuration of several \Datoms placed in a \textsc{FCC} lattice.
	To simplify, we can consider planes covered by a square lattice along $\overrightarrow{x}$ and $\overrightarrow{y}$ axes, which are interleaved with other planes along $\overrightarrow{z}$ axis.
	
	%	\subsection{Motion rules}
	
	Motion rules proposed by Piranda et al. in \cite{piranda_2016_distributed} define a list of motions that are available for a considered module and taking into account several constraints in the neighboring cells of the lattice. We will define here which conditions in terms of presence and state of modules in neighboring cells are necessary for each available motion.
	
	Let study the possible motions of a module $B$ using a module $A$ as a pivot (to simplify notations, we use the same letter to name a free cell of the lattice and the module placed in the cell if it exists).

	\begin{myDefinition}
		A motion rule is a list of tuples $(P,S)$ where $P$ is a position in the grid relative to the pivot $A$ and $S$ is a status of the cell placed at position $P$. Status $S$ can have one of the following values, or a combination of $\emptyset$ and one of the values:
		\begin{itemize}
			\item $\emptyset$, if the cell must be empty (no module at this position), 
			\item a module name, if the cell must be filled, 
			\item $def(\overrightarrow{X})$, if the cell must be filled by a deformed module, the deforming piston being oriented in the direction $\overrightarrow{X}$,
			\item $def(\overrightarrow{X},\overrightarrow{Y})$, if the cell must be filled by a module initially deformed along $\overrightarrow{X}$ axis and along $\overrightarrow{Y}$ axis at the end of the motion. 
		\end{itemize}
		
	\end{myDefinition}
	
	\begin{myTheorem}
		A motion rule is valid if all tuples of its list are validated by the current configuration.
		The Table~\ref{tab:rules} gives the list of tuples for each motion rules.
	\end{myTheorem}
	
	Table~\ref{tab:rules} gives the list of tuples for the three possible motions of $B$ with the pivot $A$ and a piston which displacement axis gives the up direction $\overrightarrow{U}$. The right direction $\overrightarrow{R}=\overrightarrow{BA} \wedge \overrightarrow{U}$ and the front direction $\overrightarrow{F}=\overrightarrow{U} \wedge \overrightarrow{R}$ are expressed relatively to the positions of $A$ and $B$. Every motion rule is defined relatively to the pivot $A$ placed at the origin of the system and $B$ at the top rear position of $A$, then two contextual following rules $\{ ( (0,0,0),A ) , (\overrightarrow{U}-\overrightarrow{F},B) \}$ can be added.
	
	\begin{myTheorem}
		\label{theo:bidirectional}
		Each displacement is bidirectional, if motion rules are valid to go from a cell $X$ to a cell $Y$, it exists a valid motion rule to go from $Y$ to $X$.
	\end{myTheorem}
	
	\begin{proof}
		If it exists a valid \textit{"Go ahead"} motion rule to go from $X$ to $Y$, as the motion constraints are symmetrical relatively to the up direction $\overrightarrow{U}$ of pivot $A$, the \textit{"Go ahead"} motion rule will be valid for a motion from $Y$ to $X$, using the same pivot $A$.
		
		\textit{"Turn left"} and \textit{"Turn left"} motion rules are symmetrical relatively to up direction $\overrightarrow{U}$ of pivot $A$. If it exists a valid \textit{"Turn left"} motion rule to go from $X$ to $Y$, it exists a valid \textit{"Turn right"} motion rule to go from $Y$ to $X$, and reciprocally.
	\end{proof}
	
	%Comprends pas -> j'enlève !
	%A mobile module can move around a pivot module connected to one of its 12 connectors, times the two pistons linked to the connectors then it gives 24 capabilities of motions.
	%Available motions are mainly surface motions, its due to the fact that a cell $K$ placed over $A$ at the top plane must be free during reconfiguration. But as shown in Figure~\ref{fig:gradient}, it is possible to move through an arch of two \Datoms height.
	
	\begin{figure}
		\center
		\includegraphics[width=0.48\textwidth]{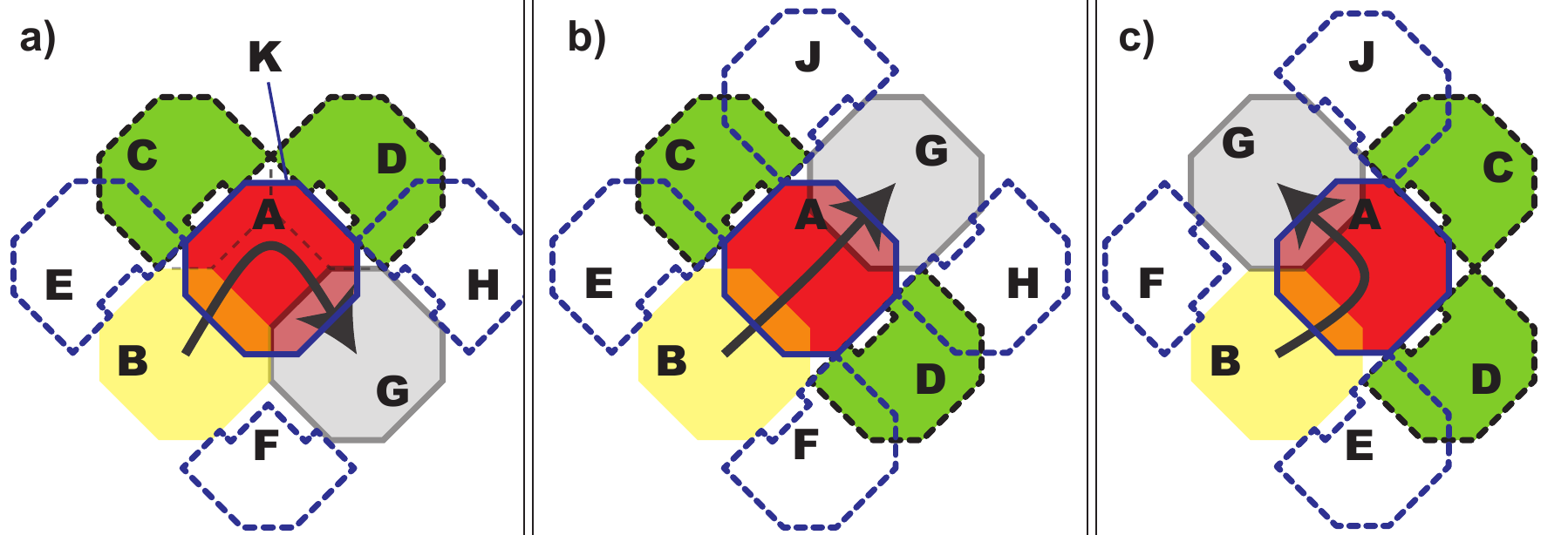}
		\caption{The three possible motions of a module B linked to a pivot A. For each motion we show the cells used by "motion rules" in the neighborhood of $A$.}
		\label{fig:constraints}%
	\end{figure}
	
	\begin{table}[ht]
		\centering
		\caption{Motion rules}
		\label{tab:rules}
		\begin{tabularx}{0.48\textwidth}{>{\setlength\hsize{0.6\hsize}}X>{\setlength\hsize{1.9\hsize}}X>{\setlength\hsize{0.5\hsize}}X} %>{\setlength\hsize{0.5\hsize}}X|>{\setlength\hsize{1.5\hsize}}X|}
			\hline
			Rule & Tuples & Cell\\ \hline
			\textit{Turn left} & $\{ (\overrightarrow{U}-\overrightarrow{R}, \emptyset)$ & Goal \\
			& $(\overrightarrow{U}+\overrightarrow{F}, \emptyset \lor def(-\overrightarrow{F})),$ & $C$\\
			& $(\overrightarrow{U}+\overrightarrow{R}, \emptyset \lor def(-\overrightarrow{R})),$ & $D$\\
			& $(2\overrightarrow{U}+\overrightarrow{R}-\overrightarrow{F}, \emptyset \lor def(-\overrightarrow{R})),$ & $E$\\ 
			& $(2\overrightarrow{U}-\overrightarrow{R}-\overrightarrow{F}, \emptyset \lor def(\overrightarrow{R},\overrightarrow{F})),$ & $F$\\ 
			& $(2\overrightarrow{U}-\overrightarrow{R}+\overrightarrow{F}, \emptyset \lor def(-\overrightarrow{F})),$ & $J$\\ 
			& $(2\overrightarrow{U}, \emptyset)\}$ & $K$\\ 
			\hline
			\textit{Turn right} & $\{ (\overrightarrow{U}+\overrightarrow{R}, \emptyset) $ & Goal \\
			& $(\overrightarrow{U}-\overrightarrow{R}, \emptyset \lor def(\overrightarrow{R})),$ & $C$\\
			& $(\overrightarrow{U}+\overrightarrow{F}, \emptyset \lor def(-\overrightarrow{F})),$ & $D$\\
			& $(2\overrightarrow{U}-\overrightarrow{R}-\overrightarrow{F}, \emptyset \lor def(\overrightarrow{R})),$ & $E$\\ 
			& $(2\overrightarrow{U}+\overrightarrow{R}-\overrightarrow{F}, \emptyset \lor def(-\overrightarrow{R},\overrightarrow{F})),$ & $F$\\ 
			& $(2\overrightarrow{U}+\overrightarrow{R}+\overrightarrow{F}, \emptyset \lor def(-\overrightarrow{F})),$ & $H$\\ 
			& $(2\overrightarrow{U}, \emptyset)\}$ & $K$\\ 
			\hline
			\textit{Go ahead} & $\{ (\overrightarrow{U}+\overrightarrow{F}), \emptyset $ & Goal \\
			& $(\overrightarrow{U}-\overrightarrow{R}, \emptyset \lor def(\overrightarrow{R})),$ & $C$\\
			& $(\overrightarrow{U}+\overrightarrow{R}, \emptyset \lor def(-\overrightarrow{R})),$ & $D$\\
			& $(2\overrightarrow{U}-\overrightarrow{R}-\overrightarrow{F}, \emptyset \lor def(\overrightarrow{R})),$ & $E$\\ 
			& $(2\overrightarrow{U}+\overrightarrow{R}-\overrightarrow{F}, \emptyset \lor def(-\overrightarrow{R})),$ & $F$\\ 
			& $(2\overrightarrow{U}+\overrightarrow{R}+\overrightarrow{F}, \emptyset \lor def(-\overrightarrow{R})),$ & $H$\\ 
			& $(2\overrightarrow{U}-\overrightarrow{R}+\overrightarrow{F}, \emptyset \lor def(\overrightarrow{R})),$ & $J$\\ 
			& $(2\overrightarrow{U}, \emptyset)\}$ & $K$\\ 
			\hline 
		\end{tabularx}
	\end{table}

	%\subsection{Possible motions}
	Figure~\ref{fig:constraints} shows an initial configuration before a motion with every cells used by at least one motion rule tuple. 
	First, we consider the plane composed of $C$ and $D$ in green, $B$ in yellow, the moving module and the goal cell $G$ in grey. 
	The pivot $A$ (drawn in red) is placed in the underneath plane. We must take into account some cells placed on the top plane ($E$, $F$, $H$ and $J$). 
	Cells with a large continuous border must be free of any module, cells with dotted border may contain a module and cells without a border must contain a module. At the top plane, the cell $K$ placed over $A$ must be free, while cells $E$, $F$, $H$ and $J$ may contain a deformed module to free the path for $B$. 
	
	The three available motions are presented separately: Figure~\ref{fig:constraints}.a) \textit{"Turn right"}, b) \textit{"Go ahead"} and c) \textit{"Turn left"}. 
	%produced by the mutual deformation of $A$ and $B$ allowing to move $B$ from the initial position to $G$. 
	%For each of these motions some cells must be free of module or filled by a module that must be deformed. 
	%These mechanical constraints has been used to defined motion rules. 
	In the case of \textit{"Turn left"} and \textit{"Turn right"} motions, if the $F$ cell is filled, the \Datom must be deformed two times during the motion. The first deformation allows $B$ to go on the top of $A$, then the deformation changes to allow $B$ to reach its final position.
	
	For example, in the first case (Figure~\ref{fig:constraints}.a), before moving module $B$ to the $G$ cell using $A$ as a pivot, we must verify that the cell $K$ on the plane on top of  $A$ is empty and then ask modules eventually placed at the $C$, $D$, $E$, $F$ and $H$ cells to deform themselves in order to free the path. 
	
	Figure~\ref{fig:constraints3D} shows the steps of the \textit{"Turn right"} motion of the module $B$:
	\begin{enumerate}[label=\alph*)]
		\item Initial configuration, \Datom $B$ plans to turn to right, it sends messages to ask $C$, $D$, $E$ and $H$ to free the path by deforming themselves.
		\item \Datoms $C$, $D$, $E$, $F$ and $H$ are deformed, $B$ can start the motion.
		\item $B$ is actuating synchronously with the pivot $A$ to create the motion.
		\item $A$ and $B$ are in the middle of the motion, they change the connectors attachment. If there is a \Datom in cell $F$, it changes its deformation to allow the final motion of $B$.
		\item $B$ reaches its final position, and asks $C$, $D$, $E$, $F$ and $H$ to release their deformation.
		\item Final configuration.
	\end{enumerate}
	
	%A border module is defined as a module that admits at least one free cell in its neighborhood. Free modules are modules that can move.
	
	\begin{figure}
		\center
		\includegraphics[width=0.48\textwidth]{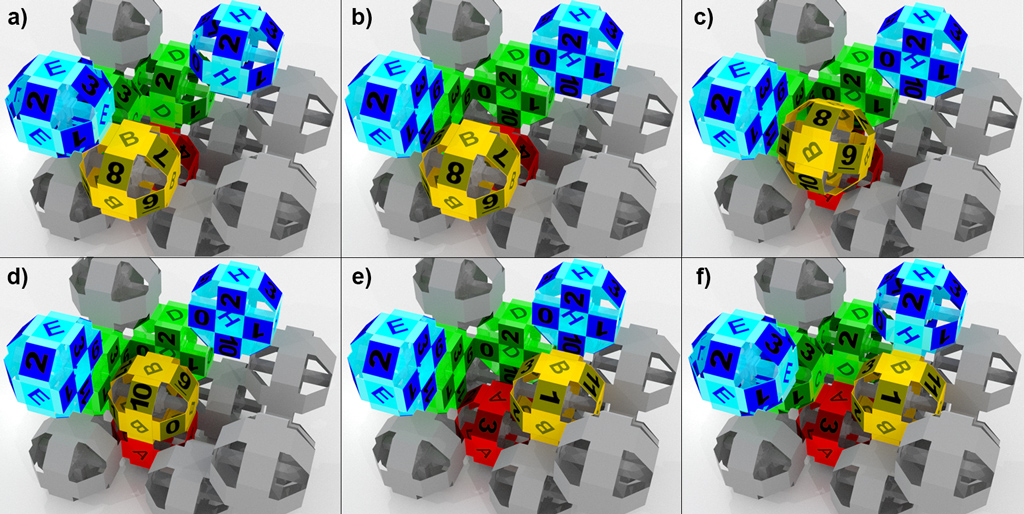}
		\caption{Some steps of displacements of a module in a constrained configuration. a) Initial configuration, b) After deformation of blocking modules. c-e) Motion steps. f) Final configuration. }
		\label{fig:constraints3D}%
	\end{figure}
	
	Figure~\ref{fig:constraints3D_2} proposes the same configuration with a new module, $F$, which must be deformed twice during the motion of $B$. In this case, Step d) is subdivided into 3 sub-steps: when $B$ reaches the position on the top of $A$, the \Datom $F$ releases its first piston and then compresses the second one, allowing $B$ to finish its motion.
	
	\begin{figure}
		\center
		\includegraphics[width=0.48\textwidth]{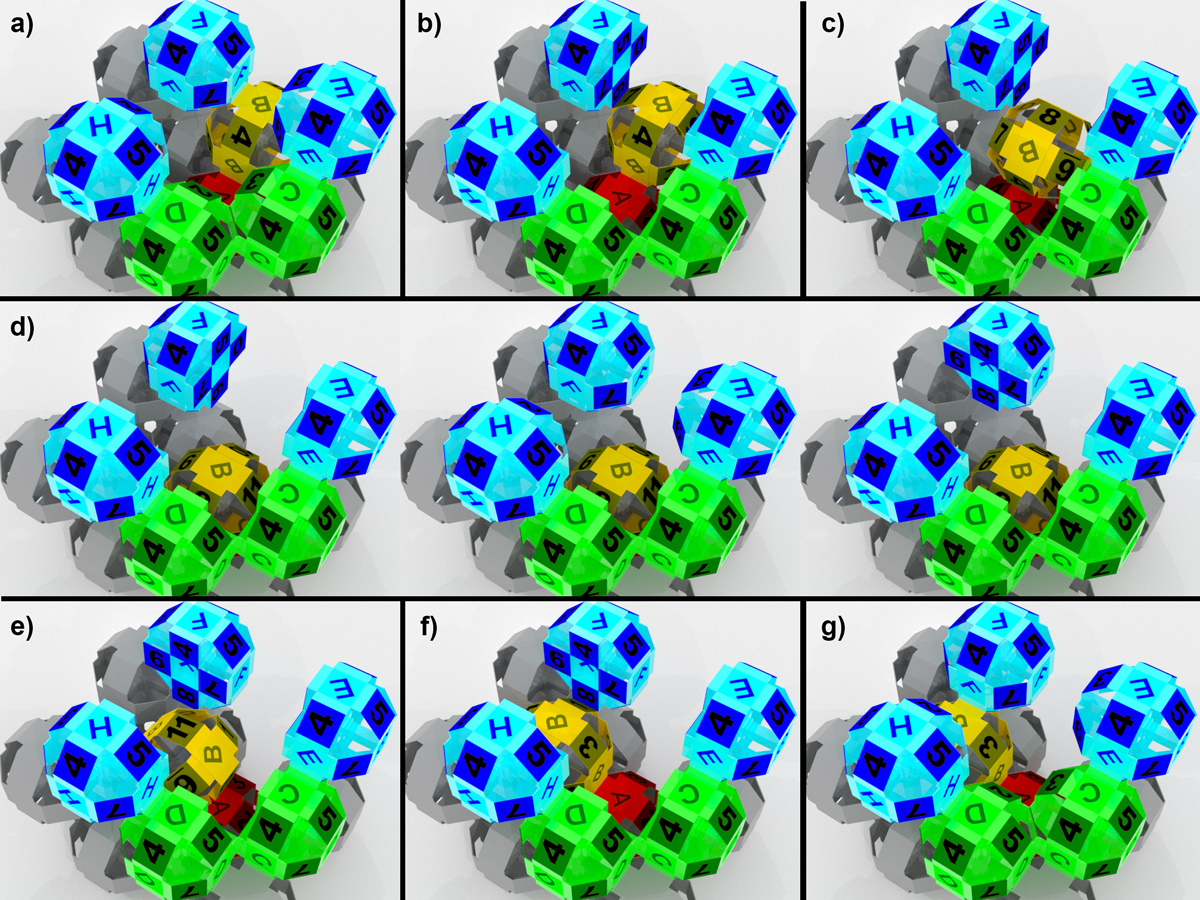}
		\caption{Some steps of displacements of a module in a constrained configuration including double deformation of module $F$.}
		\label{fig:constraints3D_2}%
	\end{figure}
	
	A particular case must be considered when $C$, $D$, $E$ or $H$ modules are only attached by one of the 4 connectors linked to the compressed piston. In this case, they must move to make the motion of $B$ possible.

	\section{Simulation}
	
	Simulations have been executed in VisibleSim~\cite{piranda_2016_visiblesim}, a modular robot simulator. 
	The goal of these experiments it to show that a \Datom can reach every free positions at the surface of a configuration only applying several unitary motions. 
	
	\subsection{Algorithms}
	We implement a first algorithm that places a \Datom at the goal cell $G$, and calculate every valid motions from this point, the reached positions are memorized in every neighbor modules. According to Theorem~\ref{theo:bidirectional} about the bidirectionality, it exists, therefore, a sequence of motions to go from each of these cells to $G$.
	We give the distance $0$ to the $G$ cell, then the distance $1$ to each cell that allows to reach $G$ cell after exactly one motion, and so on. It allows to define a gradient of distances in terms of motion to go from every reachable cells to $G$.
	
	%Nous avons pour cela implémenté un premier algorithme distribué qui place un datom à l'emplacement goal que l'on veut accéder puis calcule tous les déplacements accessibles depuis ce point, ces destination sont mémorisées au niveau de chaque datom voisin. 
	
	%D'après le theorème~\ref{theo:bidirectional}, il existe un déplacement pour aller de chacune de ces cellules à la cellule goal.
	%On associe à la cellule recevant $G$ la valeur 0, puis aux cellules accédées depuis $G$ la valeur $1$ etc. Cela permet de définir un gradient de nombre de sauts à réaliser pour allez de toutes les cellules accéssibles jusqu'à G. 

	A second algorithm (cf. Algo~\ref{alg:motion}) has been implemented to move a module $B$ (with $ID=1$) from one cell (a free cell of the border) to the goal position. 
	The \Datom $B$ calculates the list of reachable free cells from its current position. It then selects one of the cell which the minimum distance value. 
	It sends a message to all modules that must be deformed to allow its motion and, after an acknowledgement applies the motion.
	And so on, until it reaches the goal cell which distance is 0.
	
	% 	\begin{algorithm}[t]
	% 	    \caption{Gradient calculation algorithm.\label{alg:gradient}}
	% 		\SetInd{0.5em}{0.5em}
	% 		\SetKwProg{Fn}{Function}{:}{end}
	% 		\SetKwProg{Hdl}{Msg Handler}{:}{end}
	% 		\SetKwProg{EHdl}{Event Handler}{:}{end}
	% 		\SetKwFor{ForEach}{forEach}{do}{end}
	% 		\SetKwRepeat{Do}{do}{while}
	% 		\SetAlCapSkip{\MyAlgoCapSpace}
	% 		\IncMargin{\MyAlgoDecSpace}
	
	% 		\tcp{module global variables}
	% 		list$<$cell$>$ cellsList\;
	
	% 		initDistanceAllCells($\infty$)\;
	% 		distance(datom.position) $\gets$ 0\;
	
	% 		\BlankLine
	% 		\BlankLine
	% 		\Fn {gradient()} {
	% 		    d $\gets$ distance(datom.position)+1\;
	% 		    tabNextCells $\gets$ datom.getAllReachableCells()\;
	% 		    \ForEach{pos $\in$ tabNextCells} {
	% 		        \uIf{d $<$ distance(pos)} {
	% 		           distance(pos) $\gets$ d\;
	% 		           cellsList.push(pos);
	% 		        }
	% 		    }
	
	% 		    \uIf {!cellsList.empty} {
	% 		        createEvent(TeleporteModule,cellList.pop())\;
	% 		    }
	
	% 		}
	
	% 		\BlankLine
	% 		\BlankLine
	% 		\EHdl{OnTeleportationEnd()}{
	% 		    gradient()\;
	% 		}
	
	% 	\end{algorithm}
	
	\begin{algorithm}[t]
		\caption{Follow gradient.\label{alg:motion}}
		\SetInd{0.5em}{0.5em}
		\SetKwProg{Fn}{Function}{:}{end}
		\SetKwProg{Hdl}{Msg Handler}{:}{end}
		\SetKwProg{EHdl}{Event Handler}{:}{end}
		\SetKwFor{ForEach}{forEach}{do}{end}
		\SetKwRepeat{Do}{do}{while}
		\SetAlCapSkip{\MyAlgoCapSpace}
		\IncMargin{\MyAlgoDecSpace}
		\tcp{module global variables}
		int nbWaitedAnswers=0\;
		latticePosition nextPos\;
		bool isMobile=(ID==1?)\;
		module senderMobile\;
		\BlankLine
		\Fn {followGradient()} {
			tabValidRules $\gets$ getAllValidRules(datom.position)\;
			dmin $\gets \infty$ \; 
			\ForEach{rule $\in$ tabValidRules} {
				pos $\gets$ rule.finalPosition\;
				\uIf{distance(pos)$<$dmin} {
					dmin $\gets$ distance(pos)\;
					nextPos $\gets$ pos\;
					bestRule $\gets$ rule\;
				}
			}
			nbWaitedAnswers $\gets 0$\;
			\ForEach{deform $\in$ bestRule.deformationList} {
				sendMessage(deform.module,DeformMsg,\\deform.piston)\;
				nbWaitedAnswers++\;
			}
		}

		\BlankLine
		\Hdl{AckDeform(sender)} { 
			nbWaitedAnswers--\;
			\uIf {nbWaitedAnswers $=0$} {
				createEvent(DeformationModule,nextPos)\;
			}
			
		}
		\BlankLine
		\EHdl{OnDeformationEnd()}{
			\uIf{isMobile}{
				\uIf {distance(datom.position) $\ne 0$} {
					followGradient()\;
				}
			}\Else {
				sendMessage(AckDeform,senderMobile)\;
			}
		}
		\BlankLine
		\Hdl{DeformMsg(sender,piston)} { 
			senderMobile $\gets$ sender;
			createEvent(DeformationModule,piston)\;
		}
	\end{algorithm}

	\subsection{Results}
	
	\begin{figure}
		\center
		\includegraphics[width=0.48\textwidth]{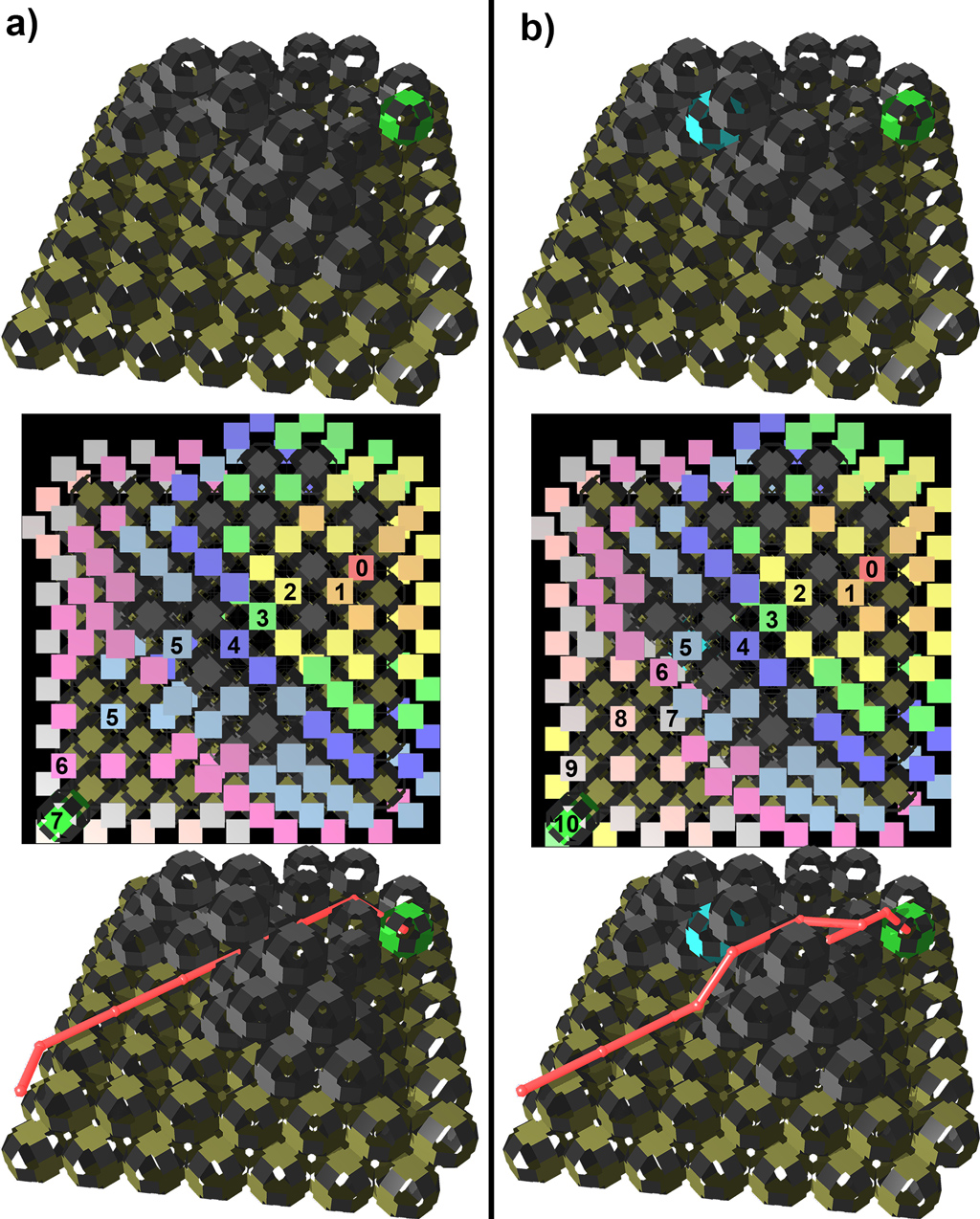}
		\caption{Simulation results of the two algorithms (gradient and motion) on two similar configurations. Distance coded by color: red:~0, orange:~1, yellow:~2, green:~3, blue:~4, cyan:~5, pink:~6, grey:~7, salmon:~8, white:~9.}
		\label{fig:gradient}%
	\end{figure}
	
	For this experiment, we construct a configuration made of 130 \Datoms. A $7 \times 7 \times 2$ box is covered by an obstacle making an arch whose hole is two \Datoms high (cf. Figure~\ref{fig:gradient}.a). And in a second time, we add a blue \Datom that reduces the size of the hole to one \Datom high only (cf. Figure~\ref{fig:gradient}.b).
	
	For these two configurations, we calculate the distance from the position $G(6,5,2)$ in the lattice (the position of the green module) to all reachable cells. The distance of these cells is represented in the second screenshot (center) where the configuration is viewed from the top. Colored square are placed at the center of the cell, the color represents the distance from the cell to the goal.
	We can observe that distances of cells at the left of the configuration are higher in the second case because the blue \Datom removes the shortcut of the arch.
	
	The third screenshot presents for the two cases the results of Algorithm~\ref{alg:motion}. The red line shows the steps of the motion of the green module from the position (0,0,2) to the goal position. In the left image, the \Datom can pass under the arch while in the second image the path goes above the obstacle.
	
	%	La première expérimentation consistent à définir simplement une boite composée de 98 datoms placés aux positions $(i,j,k) \forall i\in[1..7], j\in[1..7], k\in[2..3]$ de la lattice.
	%	On ajoute le datom $A$ à la surface de cette boite, on associe à chaque module de la boite une distance 'reseau': 

	A video that shows the deformation of the datom and some results obtained on the simulator is available on \textit{YouTube}~\footnote{YouTube video:  \url{https://youtu.be/3GZsBsvMmsU}}.

	%	La figure~\ref{fig:distance_box} montre avec un dégradé de couleurs la distance en nombre de déplacements entre la position de la cellule vide (6,5,4) de la lattice et toutes les autres cellules vides accessibles par un module supplémentaire posé sur la surface de la boite.
	
	%	\begin{figure}
	%		\center
	%		\includegraphics[width=0.48\textwidth]{figures/capture_distance01.png}
	%		\caption{Distance in term of number of motions from a given position for a 7x7x2 box. \TODO{à mettre à jour}}
	%		\label{fig:distance_box}%
	%	\end{figure}
	
	%Les petits cubes de la figure~\ref{fig:distance_box} sont placés au centre des cellules vides qui peuvent être accessibles par tout datom placé sur la surface de la boite. 
	
	%\TODO{Réaliser un modèle imprimé et montrer différents états d'animation}
	%Les photos suivantes montrent la faisabilité des parties mécaniques du modèle de datom
	
	\section{Conclusion}
	This work proposes a new model of deformable robot for programmable matter called a \Datom which allows to realize safe motions in a FCC lattice.
	The size of the components and the angular limits between these pieces are precisely detailed for the realization of a real robot.
	
	We study precisely how to implement the motion of a module in an ensemble to allow a module to step by step reach every free cell at the surface of a configuration. These motions are possible if many other modules collaborate and must synchronize their own deformation, in order to free the path for another one.
	
	Future works concern mainly the realization of actuators to add muscles to this skeleton. Many potential solutions are proposed in the paper, they must be evaluated and compared, taking into account the scalability which is a crucial point in the programmable matter domain. 
	
	\section*{ACKNOWLEDGMENT}
	This work was partially supported by the ANR (ANR-16-CE33-0022-02), the French Investissements d'Avenir program, the ISITE-BFC project (ANR-15-IDEX-03), and the EIPHI Graduate School (contract ANR-17-EURE-0002).

	\bibliographystyle{plain}  % do not change this line!
	\bibliography{datom.bib}
\end{document}